\documentclass{article}
\usepackage{nips07submit_e,times}
\usepackage{epsfig} \usepackage{color}
\usepackage{amssymb,amsfonts}
\usepackage{amsmath,amssymb,amsthm}

\theoremstyle{plain} \newtheorem{thm}{Theorem}%[section]
 \newtheorem{lem}[thm]{Lemma}

\theoremstyle{definition} \newtheorem{defn}{Definition}%[section]

\theoremstyle{remark} %[section]

\title{Anomaly Detection with Score functions based on Nearest Neighbor Graphs}

\author{
Manqi Zhao
\\
ECE Dept.\\
Boston University\\
Boston, MA 02215 \\
\texttt{mqzhao@bu.edu} \\
\And
Venkatesh Saligrama \\
ECE Dept.\\
Boston University\\
Boston, MA, 02215 \\
\texttt{srv@bu.edu} \\
}

\begin{document}

\maketitle

\begin{abstract}
We propose a novel non-parametric adaptive anomaly detection
algorithm for high dimensional data based on
score functions derived from nearest neighbor graphs on $n$-point nominal data. Anomalies
are declared whenever the score of a test sample falls below
$\alpha$, which is supposed to be the desired false alarm level. The
resulting anomaly detector is shown to be asymptotically optimal in that it is
uniformly most powerful for the specified false alarm level,
$\alpha$, for the case when the anomaly density is a mixture of the
nominal and a known density. Our
algorithm is computationally efficient, being linear in dimension and quadratic in data size. It does not require choosing complicated tuning parameters or function approximation classes and it can adapt to local structure such as local change in dimensionality.
We demonstrate the
algorithm on both artificial and real data sets in high dimensional
feature spaces.
\end{abstract}

\section{Introduction}\label{sec:Intro}
Anomaly detection involves detecting statistically significant
deviations of test data from nominal distribution. In typical
applications the nominal distribution is unknown and generally
cannot be reliably estimated from nominal training data due to a
combination of factors such as limited data size and high
dimensionality.

We propose an adaptive non-parametric method for anomaly detection
based on score functions that maps data samples to the interval
$[0,\,1]$. Our score function is derived from a K-nearest neighbor
graph (K-NNG) on $n$-point nominal data. Anomaly is declared whenever the score of a test
sample falls below $\alpha$ (the desired false alarm error). The efficacy of our method rests upon its close
connection to multivariate p-values. In statistical hypothesis
testing, p-value is any transformation of the feature space to the interval $[0,1]$ that induces a uniform distribution on the nominal data.
%such that the
%the probability (conditioned on the nominal
%hypothesis) of obtaining an observation at least as extreme as the
%one observed.
%Formally, p-value is a random variable defined over
%the feature space such that its distribution under the nominal
%distribution is uniform on the interval $[0,1]$.
When test samples
with p-values smaller than $\alpha$ are declared as anomalies, false alarm error is less than $\alpha$.

%Unfortunately, many p-values can be defined over the feature space that can be sub-optimal.
We develop a novel notion of p-values based on measures of level
sets of likelihood ratio functions. Our notion provides a characterization of the optimal
anomaly detector, in that, it is uniformly most powerful for a
specified false alarm level for the case when the anomaly density is
a mixture of the nominal and a known density. We show that our score function is
asymptotically consistent, namely, it converges to our multivariate p-value as data length approaches infinity.

Anomaly detection has been extensively studied. It is also referred
to as novelty detection \cite{Campbell, Markou}, outlier detection
\cite{Ramaswamy}, one-class classification \cite{Vert, Tax1} and
single-class classification \cite{Yaniv} in the literature.
Approaches to anomaly detection can be grouped into several
categories. In parametric approaches~\cite{basseville} the nominal
densities are assumed to come from a parameterized family and
generalized likelihood ratio tests are used for detecting deviations
from nominal. It is difficult to use parametric approaches when the distribution is unknown and data is limited.
%These approaches include the so called F-test and the
%student t-test for testing deviations of Gaussian test samples from
%the nominal variance or mean. While parametric approaches have
%proven to be effective in many applications,
A K-nearest neighbor (K-NN) anomaly detection approach is presented in \cite{Ramaswamy,Zhang2}. There
an anomaly is declared whenever the distance to the K-th nearest
neighbor of the test sample falls outside a threshold. In comparison
our anomaly detector utilizes the global information available from
the entire K-NN graph to detect deviations from the nominal. In addition it has provable optimality properties.
Learning theoretic approaches attempt to find
decision regions, based on nominal data, that separate nominal
instances from their outliers. These include one-class SVM of
Sch$\mathrm{\ddot{o}}$lkopf et. al. \cite{Scholkopf} where the basic
idea is to map the training data into the kernel space and to
separate them from the origin with maximum margin. Other algorithms
along this line of research include support vector data description
\cite{Tax},  linear programming approach \cite{Campbell}, and single
class minimax probability machine \cite{Lanckriet}.
While these approaches provide impressive computationally efficient solutions on real data, it is generally difficult to precisely relate tuning parameter choices to desired false alarm probability.

Scott and Nowak \cite{Scott1} derive decision regions based on
minimum volume (MV) sets, which does provide Type I and Type II error control.
They approximate (in appropriate function classes) level sets of the
unknown nominal multivariate density from training samples. Related
work by Hero \cite{Hero1} based on geometric entropic minimization
(GEM) detects outliers by comparing test samples to the most
concentrated subset of points in the training sample. This most
concentrated set is the $K$-point minimum spanning tree(MST) for
$n$-point nominal data and converges asymptotically to the minimum
entropy set (which is also the MV set). Nevertheless, computing
$K$-MST for $n$-point data is generally intractable. To overcome
these computational limitations \cite{Hero1} proposes heuristic
greedy algorithms based on leave-one out K-NN graph, which while
inspired by $K$-MST algorithm is no longer provably optimal. Our
approach is related to these latter techniques, namely, MV sets of
\cite{Scott1} and GEM approach of \cite{Hero1}. We develop score
functions on K-NNG which turn out to be the empirical estimates of
the volume of the MV sets containing the test point. The volume,
which is a real number, is a sufficient statistic for ensuring
optimal guarantees. In this way we avoid explicit
high-dimensional level set computation. Yet our algorithms lead to statistically optimal solutions with the ability to control false alarm and miss error probabilities.
%
%into two main approaches: density estimation
%followed by plug in estimation of ?? via variational methods; and (2) direct estimation of the level
%set using function approximation and non-parametric estimation. Since both approaches involve
%explicit approximation of high dimensional quantities, e.g.

%Philosophically,
%our approach is based on the realization that to test whether or not
%a test sample belongs to an $1-\alpha$ level set, it is sufficient
%to determine the volume (which is a real number) of the minimum
%volume level set containing the test point in a statistically
%consistent manner. This is what is achieved by our scoring function
%derived from the kNN graph.

The main features of our anomaly detector are summarized. (1) Like
\cite{Hero1} our algorithm scales linearly with dimension and
quadratic with data size and can be applied to high dimensional
feature spaces. (2) Like \cite{Scott1} our algorithm is provably
optimal in that it is uniformly most powerful for the specified
false alarm level, $\alpha$, for the case that the anomaly density
is a mixture of the nominal and any other density (not necessarily
uniform). (3) We do not require assumptions of linearity,
smoothness, continuity of the densities or the convexity of the
level sets. Furthermore, our algorithm adapts to the inherent
manifold structure or local dimensionality of the nominal density.
(4) Like \cite{Hero1} and unlike other learning theoretic approaches
such as \cite{Scholkopf,Scott1} we do not require choosing complex
tuning parameters or function approximation classes.

\section{Anomaly Detection Algorithm: Score functions based on K-NNG}\label{sec:alg}
%We first introduce the notations.
In this section we present our basic algorithm devoid of any statistical context. Statistical analysis appears in Section~\ref{sec:analysis}.
Let $S=\{x_1,x_2,\cdots,x_n\}$ be
the nominal training set of size $n$ belonging to the unit cube $[0,1]^d$. For notational convenience we use $\eta$ and $x_{n+1}$ interchangeably to denote a test
point. Our task is to declare whether the test point is consistent with nominal data or deviates from the nominal data. If the test point is an anomaly it is assumed to come from a mixture of nominal distribution underlying the training data and another known density (see Section~\ref{sec:analysis}).

Let $d(x,y)$ be a distance function denoting the distance between
any two points $x,\,y \in [0,1]^d$. For simplicity we denote the
distances by $d_{ij}= d(x_i,x_j)$. In the simplest case we assume
the distance function to be Euclidean. However, we also consider
geodesic distances to exploit the underlying manifold structure. The
geodesic distance is defined as the shortest distance on the
manifold. The {\em Geodesic Learning} algorithm, a subroutine in
Isomap~\cite{Tenenbaum,Tenenbaum2} can be used to efficiently and
consistently estimate the geodesic distances. In addition by means
of selective weighting of different coordinates note that the
distance function could also account for pronounced changes in local
dimensionality. This can be accomplished for instance through
Mahalanobis distances or as a by product of local linear
embedding~\cite{Roweis}. However, we skip these details here and
assume that a suitable distance metric is chosen.

%Based on $S$ and $\eta$, we can calculate the
%distance $d_{ij}=d(x_i,x_j)$. The distance function can range from standard Euclidean distances

%In the basic version of our algorithm,
%the distance function $d(\cdot,\cdot)$ is the Euclidean distance but
%later we will show it can be generalized to the more sophisticated
%{\em geodesic distance}.

Once a distance function is defined our next step is to form a $K$ nearest neighbor graph
(K-NNG) or alternatively an $\epsilon$ neighbor graph ($\epsilon$-NG).
%We can construct two slightly different neighborhood graph based on
%the distance metric $d_{ij}$. One is the $K$ nearest neighbor graph
K-NNG is formed by connecting each $x_i$ to the $K$ closest points
$\{x_{i_1},\cdots,x_{i_K}\}$ in $S-\{x_i\}$. We then sort the $K$ nearest distances for each $x_i$ in increasing order $d_{i,i_1}\leq
\cdots\leq d_{i,i_K}$ and denote $R_S(x_i)=d_{i,i_K}$, that is,
the distance from $x_i$ to its $K$-th nearest neighbor. We construct $\epsilon$-NG where $x_i$ and $x_j$
are connected if and only if $d_{ij}\leq \epsilon$. In this case we define
$N_S(x_i)$ as the degree of point $x_i$ in the $\epsilon$-NG. %Here
%the subscript ``$_0$'' in the notation $N_0(\cdot)$ or $R_0(\cdot)$
%reminds that the training points are drawn from distribution $f_0$.

For the simple case when the anomalous density is an arbitrary
mixture of nominal and uniform density\footnote{\tiny When the
mixing density is not uniform but, say $f_1$, the score functions
must be modified to $\hat p_K(\eta) = \frac{1}{n}\sum_{i=1}^n
\mathbb{I}{\scriptstyle\left\{\frac{1}{R_S(\eta)f_1(\eta)}\leq\frac{
1}{R_S(x_i)f_1(x_i)}\right\}}$ and $\hat p_{\epsilon} (\eta) =
\frac{1}{n}\sum_{i=1}^n
\mathbb{I}{\scriptstyle\left\{\frac{N_S(\eta)}{f_1(\eta)}\geq\frac{
N_S(x_i)}{f_1(x_i)}\right\}}$ for the two graphs K-NNG and
$\epsilon$-NNG respectively.} we consider the following two score
functions associated with the two graphs K-NNG and $\epsilon$-NNG
respectively.  The score functions map the test data $\eta$ to the
interval $[0,\,1]$.
\begin{eqnarray}
\text{K-LPE:\quad}\hat{p}_K(\eta) = \frac{1}{n}\sum_{i=1}^n
\mathbb{I}_{\{R_S(\eta)\leq R_S(x_i)\}}\label{eq:LPE1}\\
\text{$\epsilon$-LPE:\quad}\hat{p}_\epsilon(\eta) =
\frac{1}{n}\sum_{i=1}^n \mathbb{I}_{\{N_S(\eta)\geq
N_S(x_i)\}}\label{eq:LPE2}
\end{eqnarray}
where $\mathbb{I}_{\{\cdot\}}$ is the indicator function.

Finally, given a pre-defined significance level $\alpha$ (e.g.,
$0.05$), we declare $\eta$ to be anomalous if
$\hat{p_{K}}(\eta),\,\hat{p_{\epsilon}}(\eta)\leq \alpha$. We call this algorithm {\em Localized
p-value Estimation} (LPE) algorithm. This choice is motivated by its close connection to multivariate p-values(see Section \ref{sec:analysis}).

The score function K-LPE (or $\epsilon$-LPE) measures the relative
concentration of point $\eta$ compared to the training set. Section
\ref{sec:analysis} establishes that the scores for nominally
generated data is asymptotically uniformly distributed in $[0,\,1]$.
Scores for  anomalous data are clustered around $0$. Hence when scores below level
$\alpha$ are declared as anomalous the false alarm error is smaller than $\alpha$ asymptotically (since the
integral of a uniform distribution from $0$ to $\alpha$ is
$\alpha$).

%As we mentioned earlier, the Euclidean distance $d_{ij}$ can be
%replaced with the geodesic distance. To exploit the manifold
%structure of the training set $S$, the geodesic distance is defined
%as the shortest distance on the manifold. The {\sc Geodesic
%Learning} algorithm is used in \cite{Tenenbaum} as a subroutine in
%Isomap and we briefly summarize the algorithm here. The algorithm is
%accomplished in two steps. In the first step, a K'-NNG is
%constructed based on Euclidean metric. Then, in the next step, the
%shortest path between any two points are computed. Detailed
%implementation can be found in Table 1 of \cite{Tenenbaum}. It is
%proved in \cite{Tenenbaum2} that the output $\{d_{ij}\}$ converges
%to the real geodesics $d_G(x_i,x_j)$.

%\begin{center}
%\begin{minipage}{3.0 in}
%\begin{algorithm}{Geodesic-Learning}{x_1,\cdots,x_n,x_{n+1}=\eta}
%\begin{FOR}{i \= 1 \TO n+1}
%\begin{FOR}{j \= 1 \TO n+1}
%\begin{IF}{\text{$x_j$ is the Euclidean $K'$-nearest neighbor of $x_i$}}
%d_{ij} \= \|x_i-x_j\| \ELSE d_{ij} \= \infty
%\end{IF}
%\end{FOR}
%\end{FOR} \\
%\begin{FOR}{i \= 1 \TO n+1}
%\begin{FOR}{j \= 1 \TO n+1}
%\begin{FOR}{k \= 1 \TO n+1}
%d_{ij} \= \min \{d_{ij},d_{ik}+d_{kj}\}
%\end{FOR}
%\end{FOR}
%\end{FOR} \\
%\RETURN \{d_{ij}\}
%\end{algorithm}
%\end{minipage}
%\end{center}

%\begin{figure*}[!htb]
%\centering
%\begin{picture}(20,6)
%\put(-1,0){\includegraphics[scale = 0.5]{ill1.eps}}
%\put(5,0){\includegraphics[scale = 0.5]{ill1.eps}}
%\put(11,0){\includegraphics[scale = 0.5]{ill1.eps}}
%\put(3.5,0){\makebox(0,0){(a)}} \put(11,0){\makebox(0,0){(b)}}
%\end{picture}
%\caption{caption}
%\end{figure*}

\setlength{\unitlength}{1\textwidth}

\begin{figure}[!htb]
\centering
\begin{picture}(1,.35)
\put(-.04,0){\includegraphics[scale = 0.49]{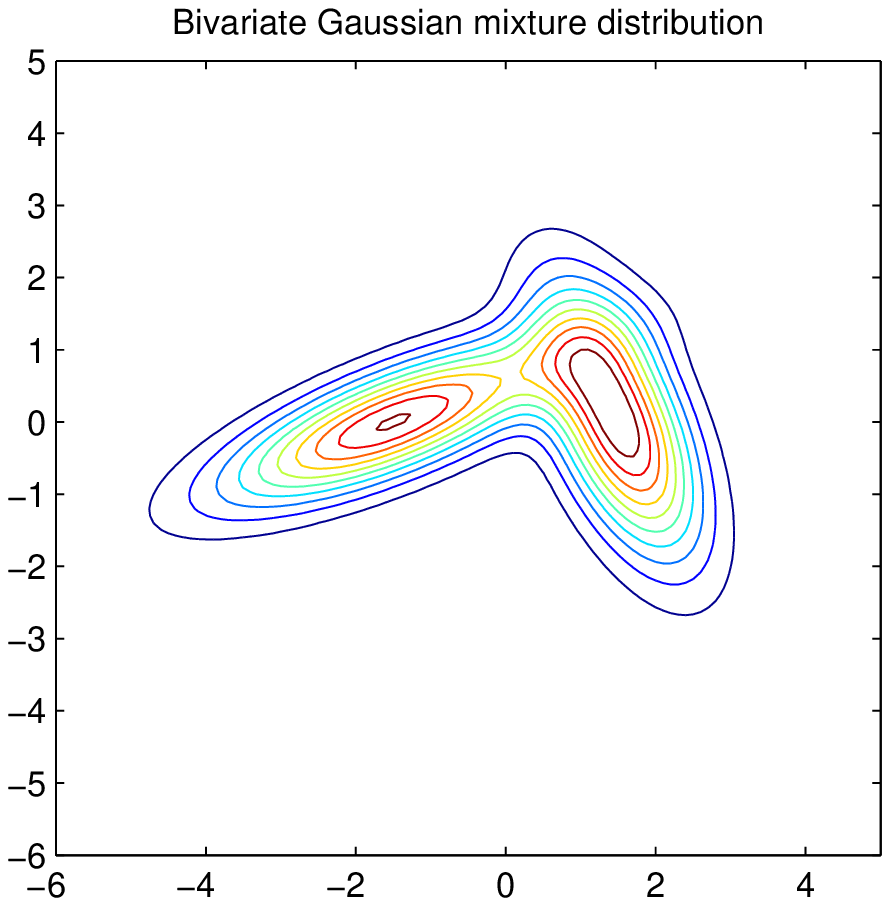}}
\put(.31,0){\includegraphics[scale = 0.49]{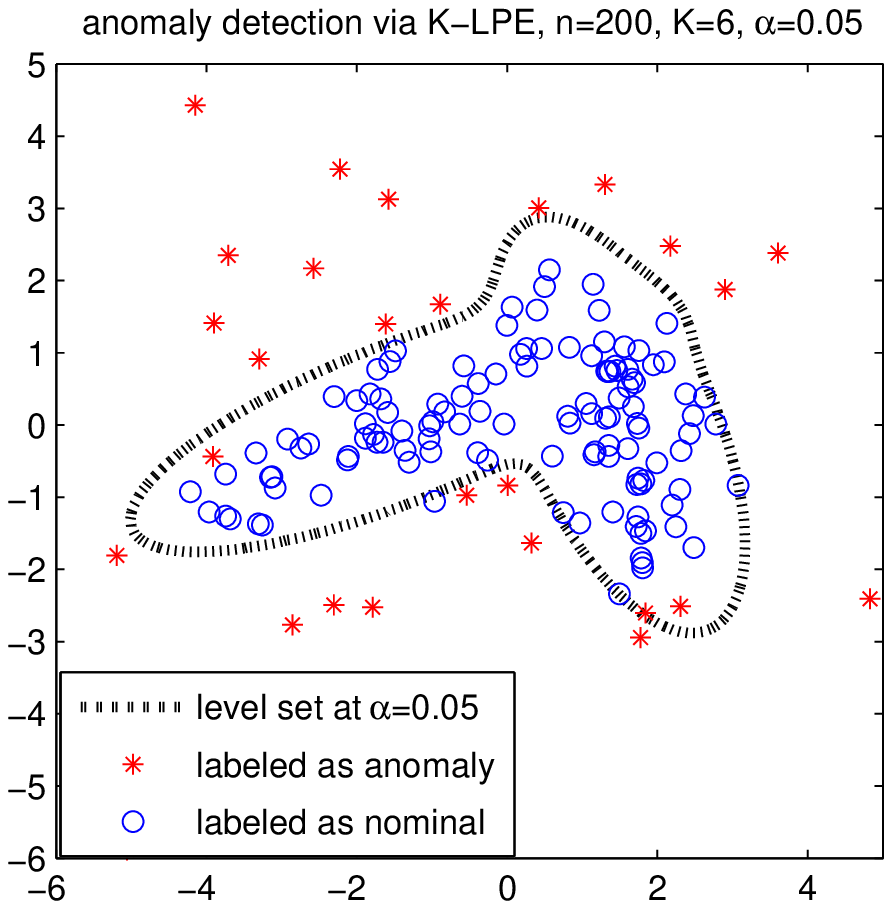}}
\put(.66,0){\includegraphics[scale = 0.49]{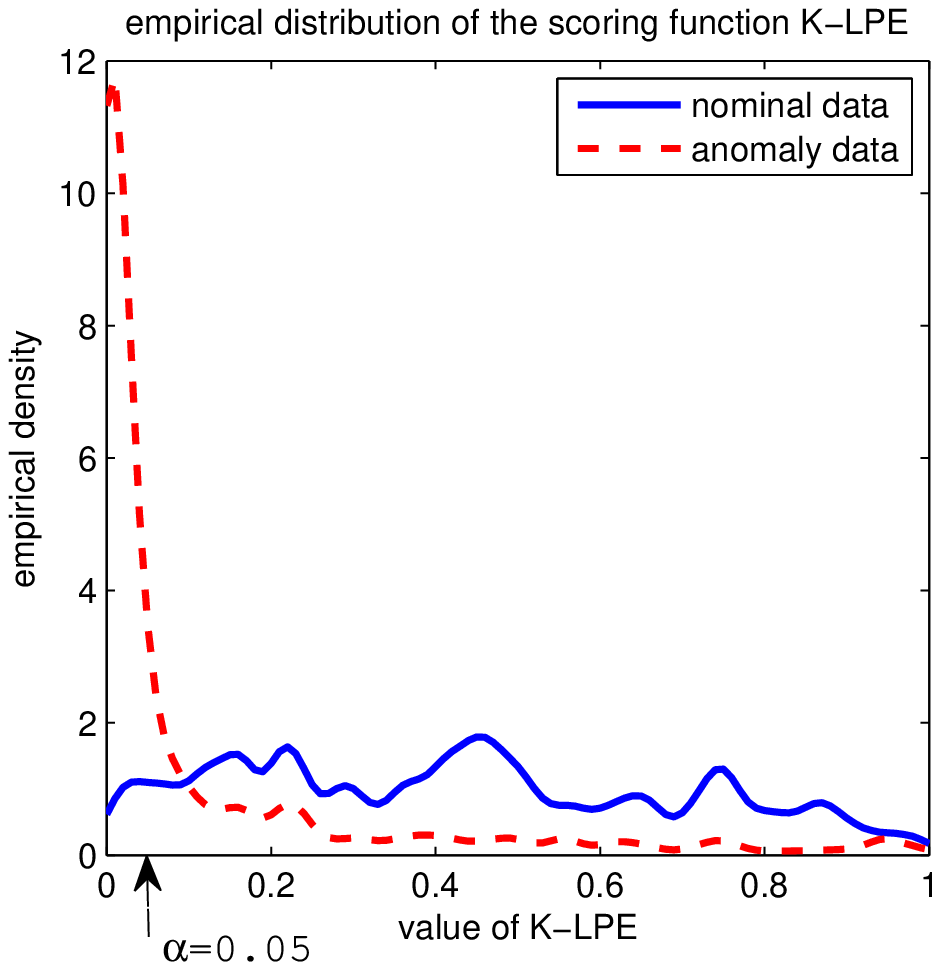}}
\end{picture}
\caption{\small \sl {\bf Left}: Level sets of the nominal bivariate
Gaussian mixture distribution used to illustrate the K-LPE
algorithm. {\bf Middle}: Results of K-LPE with $K=6$ and Euclidean distance metric for $m=150$ test points drawn from a equal mixture of 2D uniform and the (nominal) bivariate distributions. Scores for the test points are based on $200$ nominal training samples. Scores falling below a threshold level $0.05$ are declared as anomalies. The dotted contour corresponds
to the exact bivariate Gaussian density level set at level $\alpha=0.05$. {\bf Right}: The
empirical distribution of the test point scores associated with the bivariate Gaussian appear to be uniform while scores for the test points drawn from 2D uniform distribution cluster around zero.}\label{fig:illustration}
\end{figure}

Figure \ref{fig:illustration} illustrates the use of K-LPE algorithm
for anomaly detection when the nominal data is a 2D Gaussian
mixture. The middle panel of figure \ref{fig:illustration} shows the
detection results based on K-LPE are consistent with the theoretical
contour for significance level $\alpha = 0.05$. The right panel of
figure \ref{fig:illustration} shows the empirical distribution
(derived from the kernel density estimation) of the score function
K-LPE for the nominal (solid blue) and the anomaly (dashed red)
data. We can see that the curve for the nominal data is
approximately uniform in the interval $[0,\,1]$ and the curve for
the anomaly data has a peak at $0$. Therefore choosing the threshold
$\alpha=0.05$ will approximately control the Type I error within
$0.05$ and minimize the Type II error. We also take note of the inherent robustness of our algorithm. As seen from the figure (right) small changes in $\alpha$ lead to small changes in actual false alarm and miss levels.

To summarize the above discussion, our LPE algorithm has three
steps:

%\begin{enumerate}
\noindent
{\bf (1) Inputs:} Significance level $\alpha$, distance metric (Euclidean, geodesic, weighted etc.).\\
{\bf (2) Score computation: } Construct K-NNG (or $\epsilon$-NG) based on
  $d_{ij}$ and compute the score function K-LPE
  from Equation \ref{eq:LPE1} (or $\epsilon$-LPE from Equation \ref{eq:LPE2}).\\
{\bf (3) Make Decision: } Declare $\eta$ to be anomalous if
  and only if $\hat{p}_{K}(\eta)\leq \alpha$ (or $\hat{p}_{\epsilon}(\eta)\leq
  \alpha$).
%\end{enumerate}

%\begin{figure}[th]
%\begin{centering}
%\begin{minipage}[t]{.45\textwidth}
%\begin{algorithm}{$\epsilon$-LPE}{\{d_{ij}\},\{x_i\},\eta}
%\begin{FOR}{j \= 1 \TO n}
%\begin{IF}{\text{$x_j$ is the Euclidean $K'$-NN of $\eta$}}
%\rho_{j} \= \|\eta-x_j\| \ELSE \rho_{j} \= \infty
%\end{IF}
%\end{FOR} \\
%\begin{FOR}{k \= 1 \TO n}
%\rho_{j} \= \min \{\rho_{j},\rho_{k}+d_{kj}\}
%\end{FOR} \\
%N_0(\eta)\= |\{j:\rho_j\leq \epsilon\}|\\
%\hat{p}(\eta) \= 0\\
%\begin{FOR}{i \= 1\TO n}
%N_0(x_i)\= |\{j:d_{ij}\leq \epsilon\}|\\
%\begin{IF}{N_0(\eta)\geq N_0(x_i)}
%\hat{p}(\eta) \= \hat{p}(\eta)+\frac{1}{n}
%\end{IF}
%\end{FOR}\\
% \RETURN \hat{p}(\eta)
%\end{algorithm}
%\end{minipage}
%\begin{minipage}[t]{.45\textwidth}
%\begin{algorithm}{$K$-LPE}{\{d_{ij}\},\{x_i\},\eta}
%\begin{FOR}{j \= 1 \TO n}
%\begin{IF}{\text{$x_j$ is the Euclidean $K'$-NN of $\eta$}}
%\rho_{j} \= \|\eta-x_j\| \ELSE \rho_{j} \= \infty
%\end{IF}
%\end{FOR} \\
%\begin{FOR}{k \= 1 \TO n}
%\rho_{j} \= \min \{\rho_{j},\rho_{k}+d_{kj}\}
%\end{FOR} \\
%R_0(\eta)\= \rho_{j}\text{(if $x_j$ is $K$th nearest to $\eta$)}\\
%\hat{p}(\eta) \= 0\\
%\begin{FOR}{i \= 1\TO n}
%R_0(x_i)\= d_{ij}\text{(if $x_j$ is $K$th nearest to $x_i$)}\\
%\begin{IF}{R_0(\eta)\leq R_0(x_i)}
%\hat{p}(\eta) \= \hat{p}(\eta)+\frac{1}{n}
%\end{IF}
%\end{FOR}\\
% \RETURN \hat{p}(\eta)
%\end{algorithm}
%\end{minipage}
%\end{centering}
%\end{figure}

\noindent{\bf Computational Complexity:} To compute each pairwise
distance requires O(d) operations; and O($n^2d$) operations for all
the nodes in the training set. In the worst-case computing the K-NN
graph (for small $K$) and the functions $R_S(\cdot),\,N_S(\cdot)$
requires O($n^2$) operations over all the nodes in the training
data. Finally, computing the score for each test data requires
O(nd+n) operations(given that $R_S(\cdot),\,N_S(\cdot)$ have already
been computed).

\noindent {\bf Remark:} LPE is fundamentally different from
non-parametric density estimation or level set estimation schemes
(e.g., MV-set). These approaches involve explicit estimation of
high dimensional quantities and thus hard to apply in high
dimensional problems. By computing scores for each test sample we avoid high-dimensional computation. Furthermore, as we will see in the following section the scores are estimates of multivariate p-values. These turn out to be sufficient statistics for optimal anomaly detection.

\section{Theory: Consistency of LPE}\label{sec:analysis}

A statistical framework for the anomaly detection problem is
presented in this section. We establish that anomaly detection is
equivalent to thresholding p-values for multivariate data. We will
then show that the score functions developed in the previous section
is an asymptotically consistent estimator of the p-values.
Consequently, it will follow that the strategy of declaring an
anomaly when a test sample has a low score is asymptotically
optimal.

Assume that the data belongs to the d-dimensional unit cube
$[0,\,1]^d$ and the nominal data is sampled from a multivariate
density $f_0(x)$ supported on the d-dimensional unit cube
$[0,\,1]^d$. Anomaly detection can be formulated as a composite
hypothesis testing problem. Suppose test data, $\eta$ comes from a
mixture distribution, namely, $f(\eta) = (1-\pi) f_0(\eta) + \pi
f_1(\eta)$ where $f_1(\eta)$ is a mixing density supported on
$[0,1]^d$. Anomaly detection involves testing the nominal hypotheses
$H_0: \pi = 0$ versus the alternative (anomaly) $H_1 : \pi > 0$. The
goal is to maximize the detection power subject to false alarm level
$\alpha$, namely, $\mathcal{P}(\mbox{declare } H_1 \mid H_0) \leq
\alpha$.

\begin{defn}
%Assume that the measure $\mathcal{P}_1$ corresponding to the mixing density $f_1(\cdot)$ is absolutely continuous with respect to the nominal measure, $\mathcal{P}_0$, associated with the nominal density, $f_0(x)$.
Let $\mathcal{P}_0$ be the nominal probability
measure and $f_1(\cdot)$ be $\mathcal{P}_0$ measurable.
Suppose the likelihood ratio $f_1(x)/f_0(x)$ does not have \emph{non-zero} flat spots on any open ball in $[0,\,1]^d$.
Define the p-value of a data point $\eta$ as
$$p(\eta)=\mathcal{P}_0\left(x:\frac{f_1(x)}{f_0(x)}\geq\frac{f_1(\eta)}{f_0(\eta)}\right)$$
% \footnote{\tiny For simplicity we have ignored measure theoretic technicalities here. Note that for
%the definition to make sense $f_1(x)$ should be measurable w.r.t. $\mathcal{P}_0$.}.
\end{defn}
Note that the definition naturally accounts for singularities which may arise if the support of $f_0(\cdot)$ is a lower dimensional manifold. In this case we encounter $f_1(\eta)>0,\,f_0(\eta)=0$ and the p-value $p(\eta)=0$. Here anomaly is
always declared(low score).
%On the other hand when $f_1(\eta)=0,\,f_0(\eta)>0$
%then $p(\eta)=1$ and the data point is declared as nominal.
%stated in terms of the Radon-Nikodym derivative and in addition
%account for non-overlapping support of $f_1$ and $f_0$. However note
%that we do not lose any generality here. Indeed, when}.

The above formula can be thought of as a mapping of $\eta
\rightarrow [0,\,1]$. Furthermore, the distribution of $p(\eta)$
under $H_0$ is uniform on $[0,\,1]$. However, as noted in the
introduction there are other such transformations. To build
intuition about the above transformation and its utility consider
the following example.  When the mixing density is uniform, namely,
$f_1(\eta) = U(\eta)$ where $U(\eta)$ is uniform over $[0,1]^d$,
note that $\Omega_{\alpha}=\{\eta \mid p(\eta)\geq \alpha\}$ is a
density level set at level $\alpha$. It is well known (see
\cite{Scott1}) that such a density level set is equivalent to a
minimum volume set of level $\alpha$. The minimum volume set at
level $\alpha$ is known to be the uniformly most powerful decision
region for testing $H_0:\pi=0$ versus the alternative $H_1: \pi>0$
(see \cite{Hero1,Scott1}). The generalization to arbitrary $f_1$
is described next. %follows from the following theorem.
\begin{thm}
The uniformly most powerful test for testing $H_0: \pi = 0$ versus
the alternative (anomaly) $H_1 : \pi > 0$ at a prescribed level
$\alpha$ of significance $\mathcal{P}(\mbox{declare } H_1 \mid H_0)
\leq \alpha$ is:
$$
\phi(\eta) = \left \{ \begin{array}{cc} H_1, \,\, p(\eta) \leq
\alpha \\ H_0, \,\, \mbox{otherwise} \end{array} \right .
$$
\end{thm}
\begin{proof}
We provide the main idea for the proof. First, measure theoretic arguments are used to establish $p(X)$ as a random variable over $[0,\,1]$ under both nominal and anomalous distributions. %The
%proof then follows from the fact that the p-value has the following
%properties.
Next when $X \stackrel{d}{\sim} f_0$, i.e., distributed with nominal
density it follows that the random variable $p(X)\stackrel{d}{\sim} U[0,1]$. When $X \stackrel{d}{\sim} f = (1-\pi)f_0 + \pi f_1$ with $\pi>0$ the random variable, $p(X)\stackrel{d}{\sim} g$ where
$g(\cdot)$ is a monotonically decreasing PDF supported on $[0,\,1]$. Consequently, the uniformly most powerful test for a significance level $\alpha$ is to declare p-values smaller than $\alpha$ as anomalies.
\end{proof}
Next we derive the relationship between the p-values and our score
function. %Before stating the main theorems, we first introduce some
%notations and technical conditions.
By definition, $R_S(\eta)$ and
$R_S(x_i)$ are correlated because the neighborhood of $\eta$ and
$x_i$ might overlap. We modify our algorithm to simplify our analysis.
%de-correlate $R_S(\eta)$ and $R_S(x_i)$ as follows.
We assume $n$ is odd (say) and can be written as $n = 2m+1$. We divide training
set $S$ into two parts:
$$S = S_1\cap S_2=\{x_0,
x_1,\cdots,x_{m}\}\cap\{x_{m+1},\cdots,x_{2m}\}$$

We modify $\epsilon$-LPE to $\hat
p_\epsilon(\eta)=\frac{1}{m}\sum_{x_i\in S_1}
\mathbb{I}_{\{N_{S_2}(\eta)\geq
 N_{S_1}(x_i)\}}$ (or $K$-LPE
to $\hat p_K(\eta)=\frac{1}{m}\sum_{x_i\in S_1}
\mathbb{I}_{\{R_{S_2}(\eta)\leq
 R_{S_1}(x_i)\}}$). Now $R_{S_2}(\eta)$ and $ R_{S_1}(x_i)$ are independent.

Furthermore, we assume $f_0(\cdot)$ satisfies the following two
smoothness conditions:
\begin{enumerate}
  \item the Hessian matrix $H(x)$ of $f_0(x)$ is always dominated by
  a  matrix with largest eigenvalue $\lambda_M$, i.e.,
      $\exists M ~\text{s.t.}~ H(x)\preceq M~\forall x\text{ and }\lambda_{\max}(M)\leq \lambda_M$
  \item In the support of $f_0(\cdot)$, its value is always lower bounded
  by some $\beta>0$.
\end{enumerate}

We have the following theorem.
\begin{thm}\label{thm2}
Consider the setup above with the training data $\{x_i\}_{i=1}^n$ generated i.i.d. from $f_0(x)$. Let $\eta \in [0,1]^d$ be an arbitrary test sample. It follows that for a suitable choice $K$ and under the
above smoothness conditions,
$$
|\hat p_K(\eta) - p(\eta)| \stackrel{n \rightarrow
\infty}{\longrightarrow} 0 \,\,\,\mbox{almost surely, }\,\forall
\eta\in[0,1]^d
$$
\end{thm}

%%%%%%%%%%%%%%%%%%%%%%%%%%%%%%%%%%%%%%%%

For simplicity, we limit ourselves to the case when $f_1$ is uniform. The proof of Theorem~\ref{thm2} consists of two steps:
\begin{itemize}
  \item We show that the expectation
$\mathbb{E}_{S_1}\left[\hat p_\epsilon(\eta)\right]\stackrel{n
\rightarrow \infty}{\longrightarrow}p(\eta)$ (Lemma \ref{lem:mean}).
This result is then extended to K-LPE (i.e.
$\mathbb{E}_{S_1}\left[\hat p_K(\eta)\right]\stackrel{n \rightarrow
\infty}{\longrightarrow}p(\eta)$) in
Lemma \ref{lem:meanK}.
  \item Next we show that $\hat p_K(\eta)\stackrel{n
\rightarrow \infty}{\longrightarrow}\mathbb{E}_{S_1}\left[\hat
p_K(\eta)\right]$ via concentration inequality (Lemma
\ref{lem:cond}).
\end{itemize}

\begin{lem}[$\epsilon$-LPE]\label{lem:mean}
By picking $\epsilon= m^{-\frac{3}{5d}}\sqrt{\frac{d}{2\pi e}}$,
with probability at least $1-e^{-\beta m^{1/15}/2}$,
\begin{eqnarray}
l_m(\eta)\leq\mathbb{E}_{S_1}\left[\hat p_\epsilon(\eta)\right]\leq
u_m(\eta)
\end{eqnarray}
where
\begin{eqnarray*}
l_m(\eta)=\mathcal{P}_0{\left\{x: \left(f_0(\eta)-\Delta_1\right)\left(1-\Delta_2\right)\geq
\left(f_0(x)+\Delta_1\right)\left(1+\Delta_2\right)\right\}}-e^{-\beta
m^{1/15}/2}\\
u_m(\eta)=\mathcal{P}_0{\left\{x:\left(f_0(\eta)+\Delta_1\right)\left(1+\Delta_2\right)\geq
\left(f_0(x)-\Delta_1\right)\left(1-\Delta_2\right)\right\}}+e^{-\beta
m^{1/15}/2}
\end{eqnarray*}
$\Delta_1 = \lambda_M m^{-6/5d}/(2\pi e (d+2))$ and $\Delta_2 =
2m^{-1/6}$.

% the mean (with respect to $S_1$) of the modified
%$\epsilon$-LPE estimator
%$\mathbb{E}_{S_1}\left[\frac{1}{m}\sum_{x_i\in S_1}
%\mathbb{I}_{\{N_{S_2}(\eta)\geq
% N_{S_1}(x_i)\}}\right]$
%is lower and upper bounded by the pair $(l_m(\eta),u_m(\eta))$,
%where
%\begin{eqnarray}
%l_m(\eta)&=&\mathcal{P}_X{\scriptstyle\left\{\left(f_0(\eta)-\frac{\lambda_M
%m^{-\frac{6}{5d}}}{2\pi e
%(d+2)}\right)\left(1-\frac{2}{m^{1/6}}\right)\geq
%\left(f_0(X)+\frac{\lambda_M m^{-\frac{6}{5d}}}{2\pi e
%(d+2)}\right)\left(1+\frac{2}{m^{1/6}}\right)\right\}}-e^{-\beta
%m^{1/15}/2}\\
%u_m(\eta)&=&\mathcal{P}_X{\scriptstyle\left\{\left(f_0(\eta)+\frac{\lambda_M
%m^{-\frac{6}{5d}}}{2\pi e
%(d+2)}\right)\left(1+\frac{2}{m^{1/6}}\right)\geq
%\left(f_0(X)-\frac{\lambda_M m^{-\frac{6}{5d}}}{2\pi e
%(d+2)}\right)\left(1-\frac{2}{m^{1/6}}\right)\right\}}+e^{-\beta
%m^{1/15}/2}
%\end{eqnarray}
 %approximates $\frac{2}{n}\sum_{i=n/2+1}^n
%\mathbb{I}_{\left\{f_0(\eta)\geq
 %f_0(x_i)\right\}}$
\end{lem}
\begin{proof} We only prove the lower bound since the upper bound follows along similar lines.
By interchanging the expectation with the summation,
\begin{eqnarray*}
\mathbb{E}_{S_1}\left[\hat
p_\epsilon(\eta)\right]&=&\mathbb{E}_{S_1}\left[\frac{1}{m}\sum_{x_i\in
S_1} \mathbb{I}_{\{N_{S_2}(\eta)\geq
 N_{S_1}(x_i)\}}\right]\\
&=&\frac{1}{m}\sum_{x_i\in
S_1}\mathbb{E}_{x_i}\mathbb{E}_{S_1\setminus x_i}\left[
\mathbb{I}_{\{N_{S_2}(\eta)\geq
 N_{S_1}(x_i)\}}\right]\\
 &=&
 \mathbb{E}_{x_1}[\mathcal{P}_{S_1\setminus x_1}(N_{S_2}(\eta)\geq
 N_{S_1}(x_1))]
\end{eqnarray*}
where the last inequality follows from the symmetric structure of
$\{x_0,x_1,\cdots,x_m\}$.

Clearly the objective of the proof is to show
$\mathcal{P}_{S_1\setminus x_1}(N_{S_2}(\eta)\geq
 N_{S_1}(x_1))\stackrel{n
\rightarrow \infty}{\longrightarrow}
\mathbb{I}_{\left\{f_0(\eta)\geq
 f_0(x_1)\right\}}$. Skipping technical details, this can be
 accomplished in two steps. (1) Note that $N_S(x_1)$ is a
 binomial random variable with success probability
 $q(x_1):=\int_{B_\epsilon}f_0(x_1+t)\text{d}t$. This relates $\mathcal{P}_{S_1\setminus x_1}(N_{S_2}(\eta)\geq
 N_{S_1}(x_1))$ to $\mathbb{I}_{\left\{q(\eta)\geq
 q(x_1)\right\}}$. (2) We relate $\mathbb{I}_{\left\{q(\eta)\geq
 q(x_1)\right\}}$ to $\mathbb{I}_{\left\{f_0(\eta)\geq
 f_0(x_1)\right\}}$ based on the function smoothness condition. The
 details of these two steps are shown in the below.

Note that $N_{S_1}(x_1)\sim \text{Binom}({m},q(x_1))$. By Chernoff
bound of binomial distribution, we have
\begin{eqnarray*}
%\mathcal{P}_{S_3}(|N_{S_3}(\eta)-mq_\eta|\geq
%\delta)&\leq&
%2\exp\left(-\frac{\delta^2}{2n q_\eta}\right)\\
\mathcal{P}_{S_1\setminus x_1}(N_{S_1}(x_1)-m q(x_1)\geq \delta)\leq
e^{-\frac{\delta^2}{2m q(x_1)}}
\end{eqnarray*}
that is, $N_{S_1}(x_1)$ is concentrated around $mq(x_1)$.  This
implies,
\begin{eqnarray}\label{eq:Pr_In} \mathcal{P}_{S_1\setminus x_1}(N_{S_2}(\eta)\geq
 N_{S_1}(x_1))\geq \mathbb{I}_{\left\{N_{S_2}(\eta)\geq m q(x_1)+\delta_{x_1}\right\}}
 -e^{-\frac{\delta_{x_1}^2}{2m
 q(x_1)}}\end{eqnarray}
%and
%\begin{eqnarray}\label{eq:Pr_In2}\mathcal{P}_{S_1\setminus x_1}(N_{S_2}(\eta)\geq
% N_{S_1}(x_1))\leq \mathbb{I}_{\left\{N_{S_2}(\eta)\leq m q_{x_1}-\delta_{x_1}\right\}}
% +e^{-\frac{\delta_{x_1}^2}{2m
% q_{x_1}}}\end{eqnarray}
%We will focus on further lower-bounding the RHS of equation
%(\ref{eq:Pr_In}). Similar manipulation can be applied to equation
%(\ref{eq:Pr_In2}) and therefore is omitted.

We choose $\delta_{x_1}=q(x_1)m^\gamma(\gamma \text{ will be
specified later})$ and reformulate equation~(\ref{eq:Pr_In}) as
\begin{eqnarray}\label{eq:Pr_In3}
\mathcal{P}_{S_1\setminus x_1}(N_{S_2}(\eta)\geq N_{S_1}(x_1))\geq
\mathbb{I}_{\left\{\frac{N_{S_2}(\eta)}{m\text{Vol}(B_\epsilon)}\geq
\frac{q(x_1)}{\text{Vol}(B_\epsilon)}\left(1+\frac{2}{m^{1-\gamma}}\right)\right\}}
 -e^{-\frac{q(x_1) m^{2\gamma-1}}{2}}\end{eqnarray}

Next, we relate $q(x_1)(\text{or }
\int_{B_\epsilon}f_0(x_1+t)\text{d}t)$ to $f_0(x_1)$ via the
Taylor's expansion and the smoothness condition of $f_0$,
\begin{eqnarray} \label{eq:step2} \left|\frac{\int_{B_\epsilon}f_0(x_1+t)\text{d}t}{\text{Vol}(B_\epsilon)}-
f_0(x_1)\right| \leq
\frac{\lambda_M}{2}\cdot\frac{1}{\text{Vol}(B_\epsilon)}\int_{B_\epsilon}\|t\|^2\text{d}t
=\frac{\lambda_M \epsilon^2}{2d(d+2)}
\end{eqnarray}
and then equation~(\ref{eq:Pr_In3}) becomes
\[\mathcal{P}_{S_1\setminus x_1}(N_{S_2}(\eta)\geq N_{S_1}(x_1))\geq
\mathbb{I}_{\left\{\frac{N_{S_2}(\eta)}{m\text{Vol}(B_\epsilon)}\geq
\left(f_0(x_1)+\frac{\lambda_M
\epsilon^2}{2d(d+2)}\right)\left(1+\frac{2}{m^{1-\gamma}}\right)\right\}}
 -e^{-\frac{q(x_1) m^{2\alpha-1}}{2}}\]
%Denote the indicator function
%$\mathbb{I}_{\left\{\frac{N_{S_3}(\eta)}{m\text{Vol}(B_\epsilon)}\geq
%\left(f_0(x_i)+\frac{\lambda_M
%\epsilon^2}{2d(d+2)}\right)\left(1+\frac{2}{m^{1-\alpha}}\right)\right\}}$
%as a binary random variable $Z_i$ and take average of the above
%equation in the set $S_2$,
%
%\begin{eqnarray}
%\frac{1}{m}\sum_{x_i\in S_2}\mathcal{P}_{S_1}(N_{S_3}(\eta)\geq
%N_{S_1}(x_i))\geq \frac{1}{m}\sum_{x_i\in S_2}
%Z_i-2e^{-\frac{q_{x_i} m^{2\alpha-1}}{2}}
%\end{eqnarray}
%By applying Hoeffding's inequality to the RHS of the above equation,
%we have that with probability at least $1-\delta$,
%\begin{eqnarray*}&&\frac{1}{m}\sum_{x_i\in S_2}\mathcal{P}_{S_1}(N_{S_3}(\eta)\geq N_{S_1}(x_i))\\&\geq&
%\mathcal{P}_X{\left\{\frac{N_{S_3}(\eta)}{m\text{Vol}(B_\epsilon)}\geq
%\left(f_0(X)+\frac{\lambda_M
%\epsilon^2}{2d(d+2)}\right)\left(1+\frac{2}{m^{1-\alpha}}\right)\right\}}
%-2e^{-\frac{q_{x_i}
%m^{2\alpha-1}}{2}}-\sqrt{\frac{\log(2/\delta)}{m}}\end{eqnarray*}

By applying the same steps to $N_{S_2}(\eta)$ as equation
\ref{eq:Pr_In} (Chernoff bound) and equation \ref{eq:step2}
(Taylor's explansion), we have with probability at least
$1-e^{-\frac{q(\eta) m^{2\alpha-1}}{2}}$,
%\begin{eqnarray*}&&\mathbb{E}_{x_1}[\mathcal{P}_{S_1\setminus x_1}(N_{S_2}(\eta)\geq N_{S_1}(x_1))]\\&\geq&
%\mathcal{P}_{x_1}{\scriptstyle{\left\{\left(f_0(\eta)-\frac{\lambda_M
%\epsilon^2}{2d(d+2)}\right)\left(1-\frac{2}{m^{1-\gamma}}\right)\geq
%\left(f_0(x_1)+\frac{\lambda_M
%\epsilon^2}{2d(d+2)}\right)\left(1+\frac{2}{m^{1-\gamma}}\right)\right\}}}
%-e^{-\frac{q(x_1) m^{2\alpha-1}}{2}}\end{eqnarray*}
{\small
$$
\mathbb{E}_{x_1}[\mathcal{P}_{S_1\setminus x_1}(N_{S_2}(\eta)\geq N_{S_1}(x_1))]\geq
\mathcal{P}_{x_1}{\scriptstyle{\left\{\left(f_0(\eta)-\frac{\lambda_M
\epsilon^2}{2d(d+2)}\right)\left(1-\frac{2}{m^{1-\gamma}}\right)\geq
\left(f_0(x_1)+\frac{\lambda_M
\epsilon^2}{2d(d+2)}\right)\left(1+\frac{2}{m^{1-\gamma}}\right)\right\}}}
-e^{-\frac{q(x_1) m^{2\alpha-1}}{2}}
$$}

Finally, by choosing $\epsilon^2 = m^{-\frac{6}{5d}}\cdot
\frac{d}{2\pi e}$ and $\gamma = 5/6$, we prove the lemma.
\end{proof}

\begin{lem}[$K$-LPE]\label{lem:meanK}
By picking $K =\left(1-2m^{-1/6}
\right)m^{2/5}\left(f_0(\eta)-\Delta_1\right)$, with probability at
least $1-e^{-\beta m^{1/15}/2}$,
\begin{eqnarray}
l_m(\eta)\leq\mathbb{E}_{S_1}\left[\hat p_K(\eta)\right]\leq
u_m(\eta)
\end{eqnarray}
\end{lem}
\begin{proof}
%The proof is very similar to the proof to Lemma \ref{lem:mean} and
%is omitted.

The proof is very similar to the proof to Lemma \ref{lem:mean} and
we only give a brief outline here. Now the objective is to show
$\mathcal{P}_{S_1\setminus x_1}(R_{S_2}(\eta)\leq
 R_{S_1}(x_1))\stackrel{n
\rightarrow \infty}{\longrightarrow}
\mathbb{I}_{\left\{f_0(\eta)\geq
 f_0(x_1)\right\}}$.The basic idea is to use the
result of Lemma \ref{lem:mean}. To accomplish this, we note that
$\{R_{S_2}(\eta)\leq
 R_{S_1}(x_1)\}$ contains the events $\{N_{S_2}(\eta)\geq K\}\cap\{N_{S_1}(x_1)\leq
 K\}$, or equivalently
 \begin{eqnarray}
\{N_{S_2}(\eta)-q(\eta) m\geq K-q(\eta)
m\}\cap\{N_{S_1}(x_1)-q(x_1)m\leq
 K-q(x_1)m\}
 \end{eqnarray}
%where $q_{x}=\int_{B_\epsilon}f_0(x+t)\text{d}t$is defined with the
%$\epsilon= m^{-\frac{3}{5d}}\sqrt{\frac{d}{2\pi
% e}}$ (c.f. the proof to
%Lemma \ref{lem:mean}).

By the tail probability of Binomial distribution, the probability of
the above two events converges to 1 exponentially fast if $K-q(\eta)
m<0$ and $K-q(x_1)m> 0$. By using the same two-step bounding
techniques developed in the proof to Lemma \ref{lem:mean}, these two
inequalities are implied by
\begin{eqnarray*}
K-m^{2/5}\left(f_0(\eta)-\Delta_1\right)<
0\,~\text{and}\,~K-m^{2/5}\left(f_0(x_1)+\Delta_1\right)> 0
\end{eqnarray*}
Therefore if we choose $K=\left(1-2m^{-1/6}
\right)m^{2/5}\left(f_0(\eta)- \Delta_1\right)$, we have with
probability at least $1-e^{-\beta m^{-1/15}/2}$,
\begin{eqnarray*}\mathcal{P}_{S_1\setminus x_1}(R_{S_2}(\eta)\leq
R_{S_1}(x_1))\geq
\mathbb{I}_{{\left\{\left(f_0(\eta)-\Delta_1\right)\left(1-\Delta_2\right)\geq
\left(f_0(x_1)+\Delta_1\right)\left(1+\Delta_2\right)\right\}}}
-e^{-\beta m^{-1/15}/2}\end{eqnarray*}

\end{proof}

\noindent{\bf Remark:} Lemma \ref{lem:mean} and Lemma \ref{lem:meanK} were proved with specific choices for $\epsilon$ and $K$.
However, $\epsilon$
and $K$ can be chosen in a range of values, but will lead to different lower and upper bounds. We will show in Section \ref{sec:ex} through simulations that our LPE algorithm is generally robust to choice of parameter $K$.

\begin{lem}\label{lem:cond}
Suppose $K=cm^{2/5}$ and denote
$\hat{p}_K(\eta)=\frac{1}{m}\sum_{x_i\in S_1}
\mathbb{I}_{\{R_{S_2}(\eta)\leq
 R_{S_1}(x_i)\}}$. We have
\[\mathcal{P}_0\left(|\mathbb{E}_{S_1}\left[\hat{p}_K(\eta)\right]-\hat{p}_K(\eta)|>\delta\right)\leq 2e^{-\frac{2\delta^2m^{1/5}}{c^2\gamma_d^2}}\]
where $\gamma_d$ is a constant and is defined as the minimal number
of cones centered at the origin of angle $\pi/6$ that cover
$\mathbb{R}^d$.
\end{lem}
\begin{proof}
We can not apply Law of Large Number in this case because
$\mathbb{I}_{\{R_{S_2}(\eta)\leq
 R_{S_1}(x_i)\}}$ are correlated. Instead, we need to use the more generalized concentration-of-measure inequality such as MacDiarmid's inequality\cite{Mac}. Denote
$F(x_0,\cdots,x_m)=\frac{1}{m}\sum_{x_i\in
S_1}\mathbb{I}_{\{R_{S_2}(\eta)\leq R_{S_1}(x_i)\}}$. From Corollary
11.1 in \cite{Lugosi},
\begin{eqnarray}\label{eq:mac}
\sup_{x_0,\cdots,x_m,x_i'}|F(x_0,\cdots,x_i,\cdots,x_m)-F(x_0,\cdots,x_i',\cdots,x_n)|\leq
K\gamma_d/m\end{eqnarray}

%For $\epsilon$-neighborhood, suppose we replace one element of
%$\{x_1,\cdots,x_m\}$ with a new sample $x'$ from distribution $f_0$.
%The number of changed $N_{S_1}(x_i)$ is upper bounded by $Cm^{2/5}$.
%Therefore, the difference of the old $F(\cdot)$ value with the
%perturbed new $F(\cdot)$ value is upper bounded by
%$Cm^{2/5}/m=Cm^{-3/5}$.

Then the lemma directly follows from applying McDiarmid's
inequality.
\end{proof}

Theorem \ref{thm2} directly follows from the combination of Lemma
\ref{lem:meanK} and Lemma \ref{lem:cond} and a standard application of the first Borel-Cantelli lemma. We have used Euclidean distance
in Theorem \ref{thm2}. When the support of $f_0$ lies on a lower dimensional manifold (say $d'< d$) adopting the geodesic metric leads to faster convergence. It turns out that $d'$ replaces $d$ in the expression for
$\Delta_1$ in Lemma 3.

\section{Experiments}\label{sec:ex}
We apply our method on both artificial and real-world data. Our
method enables plotting the entire ROC curve by varying the thresholds on our scores. %In comparison \cite{Hero1} the
%ROC curves is also derived but it is fundamentally restricted to the
%region where the type I error $\alpha$ is extremely small.

To test the sensitivity of K-LPE to parameter changes, we first run
$K$-LPE on the benchmark artificial data-set {\tt Banana}
\cite{banana} with $K$ varying from $2$ to $12$. {\tt Banana}
dataset contains points with their labels($+1$ or $-1$). We randomly
pick $109$ points with $+1$ label and regard them as the nominal
training data. The test data comprises of 108 $+1$ data and 183 $-1$
data (ground truth) and the algorithm is supposed to predict $+1$
data as ``nominal'' and $-1$ data as ``anomaly''. See Figure
\ref{fig:robust}(a) for the configuration of the training points and
test points. Scores computed for test set using Equation
\ref{eq:LPE1} is oblivious to true $f_1$ density ($-1$ labels).
Euclidean distance metric is adopted for this example.

False alarm (also called false positive) is defined as the
percentage of nominal points that are predicted as anomaly by the
algorithm. To control false alarm at level $\alpha$, point with
score $< \alpha$ is predicted as anomaly. Empirical false alarm and
true positives (percentage of anomalies declared as anomaly) can be
computed from ground truth. We vary $\alpha$ to obtain the empirical
ROC curve. We follow this procedure for all the other experiments in
this section. We are relatively insensitive to $K$ as shown in
Figure \ref{fig:robust}(b).

For comparison we plot the empirical ROC curve of the one-class SVM
of \cite{Scholkopf}. There are two tuning parameters in OC-SVM ---
bandwidth $c$ (we use RBF kernel) and $\nu\in(0,1)$ (which is
supposed to control FA). Note that training data {\em does not}
contain $-1$ labels and this implies we can never make use of $-1$
labels to cross-validate, or, to optimize over the choice of pair
$(c,\nu)$. In our OC-SVM implementation, by following the same
procedure, we can obtain the empirical ROC curve by varying $\nu$
but {\em fixing} a certain bandwidth $c$. Finally we iterated over
different $c$ to obtain the best (in terms of AUC) ROC curve and it
turns out to be $c=1.5$. Fixing $c$ for entire ROC is equivalent to
fixing $K$ in our score function. Note that in real practice what
can be done is even worse than this implementation because there is
also no natural way to optimize over $c$ without being revealed the
$-1$ labels.

In Figure~\ref{fig:robust}(b), we can see that our algorithm is
consistently better than one-class SVM on the {\tt Banana} dataset.
Furthermore, we found that choosing suitable tuning parameters to
control false alarms is generally difficult in the one-class SVM
approach. In our approach if we set $\alpha = 0.05$ we get empirical
$FA = 0.06$ and for $\alpha = 0.08$, empirical $FA = 0.09$. For
OC-SVM we can not see any natural way of picking $c$ and $\nu$ to
control FA rate based only on training data.

%\begin{figure}[t]
%\begin{centering}
%\begin{minipage}[t]{.45\textwidth}
%\includegraphics[width = 2.5 in]{banana.eps}\\
%\begin{minipage}[t]{2.5 in}
%\centering (a)
%\end{minipage}
%\end{minipage}
%\begin{minipage}[t]{.4\textwidth}
%\includegraphics[width = 2.5 in]{banana2.eps}\\
%\begin{minipage}[t]{2.5 in}
%\centering (b)
%\end{minipage}
%\end{minipage}
%\caption{Anomaly detection on the {\tt banana} dataset; (a) the
%layout of training set (all with positive label) and unlabeled test
%set; (b) the empirical ROC curve based on our method and the
%one-class SVM developed in \cite{Scholkopf}.} \label{fig:bana}
%\end{centering}
%\end{figure}

\begin{figure}[t]
\begin{centering}
%\begin{minipage}[t]{.25\textwidth}
%\includegraphics[width = 1.0 in]{banana.eps}\\
%\begin{minipage}[t]{1.0 in}
%\centering (a)
%\end{minipage}
%\end{minipage}
\begin{minipage}[t]{.48\textwidth}
\includegraphics[width = 1\textwidth]{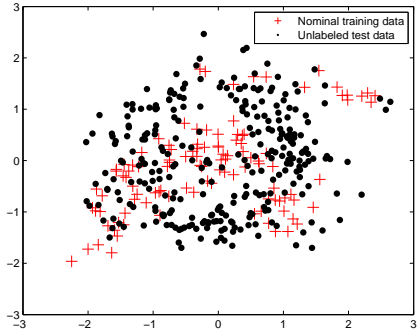}\\
\makebox[7 cm]{(a) Configuration of {\tt banana} data}
\end{minipage}
\begin{minipage}[t]{.48\textwidth}
\includegraphics[width = 1\textwidth]{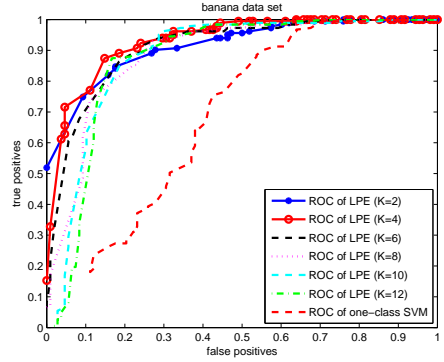}\\
\makebox[7 cm]{(b) SVM vs. K-LPE for Banana Data}
\end{minipage}
\caption{\small \sl Performance Robustness of LPE;(a) The
configuration of the nominal training points (red `+') and unlabeled
test points (black ` $\bullet$') for the {\tt banana} dataset
\cite{banana}; (b) Empirical ROC curve of $K$-LPE on the {\tt
banana} dataset
 with $K=2,4,6,8,10,12$ (with $n=400$) vs the empirical
ROC curve of one class SVM developed in \cite{Scholkopf}.}
\label{fig:robust}
\end{centering}
\end{figure}

\begin{figure}[t]
\begin{centering}
%\begin{minipage}[t]{.25\textwidth}
%\includegraphics[width = 1.0 in]{banana.eps}\\
%\begin{minipage}[t]{1.0 in}
%\centering (a)
%\end{minipage}
%\end{minipage}
\begin{minipage}[t]{.48\textwidth}
\includegraphics[width = 1\textwidth]{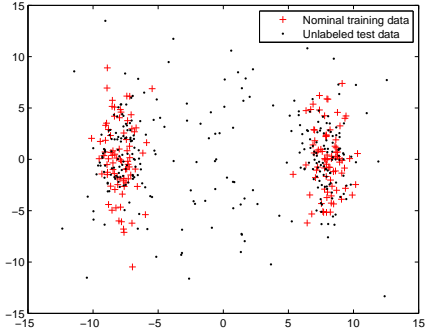}\\
\makebox[7 cm]{(a) Configuration of data}
\end{minipage}
\begin{minipage}[t]{.48\textwidth}
\includegraphics[width = 1\textwidth]{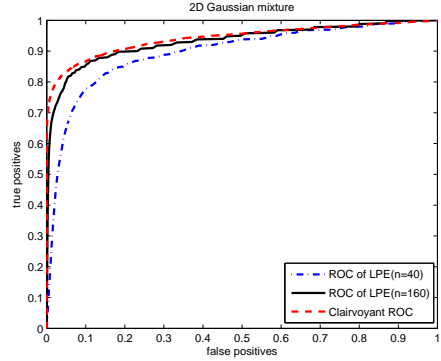}\\
\makebox[7 cm]{(b) Clairvoyant vs. K-LPE}
\end{minipage}
\caption{\small \sl Clairvoyant ROC curve vs. K-LPE; (a)
Configuration of the nominal training points and unlabeled test
points for the data given by Equation \ref{eq:f0f1}; (b) Averaged
(over $15$ trials) empirical ROC curves of $K$-LPE algorithm vs
clairvoyant ROC curve (when $f_0$ is given by Equation
\ref{eq:f0f1}) for $K=6$ and for different values of $n$
($n=40,160$).} \label{fig:clair}
\end{centering}
\end{figure}

%In the next experiment, we want to test the sensitivity of LPE to
%parameter changes. Figure \ref{fig:robust}(a) shows the results of
%$K$-LPE on the {\tt banana} dataset with $K$ varying from $2$ to
%$12$. $K=4$ and $K=6$ are better choices compared to other values.
%However, all the ROC curves are close to each other and it means our
%algorithm is insensitive (robust) to the choice of $K$.
In Figure \ref{fig:clair}, we apply our $K$-LPE to another 2D
artificial example where the nominal distribution $f_0$ is a mixture
Gaussian and the anomalous distribution is very close to uniform
(see Figure \ref{fig:clair}(a) for their configuration):
\begin{eqnarray}\label{eq:f0f1}
f_0\sim\frac{1}{2}\mathcal{N}\left(\begin{bmatrix}8\\0\end{bmatrix},\begin{bmatrix}1&
0\\0&
9\end{bmatrix}\right)+\frac{1}{2}\mathcal{N}\left(\begin{bmatrix}-8\\0\end{bmatrix},\begin{bmatrix}1&
0\\0& 9\end{bmatrix}\right), \quad
f_1\sim\mathcal{N}\left(0,\begin{bmatrix}49& 0\\0&
49\end{bmatrix}\right)\end{eqnarray} In this example, we can exactly
compute the optimal ROC curve. We call this curve the {\em
Clairvoyant ROC} (the red dashed curve in Figure
\ref{fig:clair}(b)). The other two curves are averaged (over $15$
trials) empirical ROC curves with respect to different sizes of
training sample ($n=40,160$) for $K=6$. Larger $n$ results in better
ROC curve. We see that for a relatively small training set of size
$160$ the average empirical ROC curve is very close to the
clairvoyant ROC curve.

Next, we ran LPE on three real-world datasets: {\tt Wine}, {\tt
Ionosphere}\cite{Asuncion+Newman:2007} and MNIST US Postal Service
({\tt USPS}) database of handwritten digits. The procedure and setup
of the experiments is almost the same as the that of the {\tt
Banana} data set. However, there are two differences. (1) If the
number of different labels is greater than two, we always treat
points with one particular label as nominal($+1$) and regard the
points with other labels as anomalous($-1$). For example, for the
{\tt USPS} dataset, we regard instances of digit $0$ as nominal
training samples and instances of digits $1,\cdots,9$ as anomaly.
(2) For high dimensional data set, the data points are normalized to
be within $[0,1]^d$ and we use geodesic distance
\cite{Tenenbaum}(instead of Euclidean distance) as the input to LPE.

%We need to confess that this experiment is not a ``real" anomaly
%detection experiment. However, this setup gives us the convenience
%to count the type I and type II errors.

\setlength{\unitlength}{1\textwidth}

\begin{figure*}[!htb]
\centering
\begin{picture}(1,.30)
\put(-0.03,0){\includegraphics[scale = 0.33]{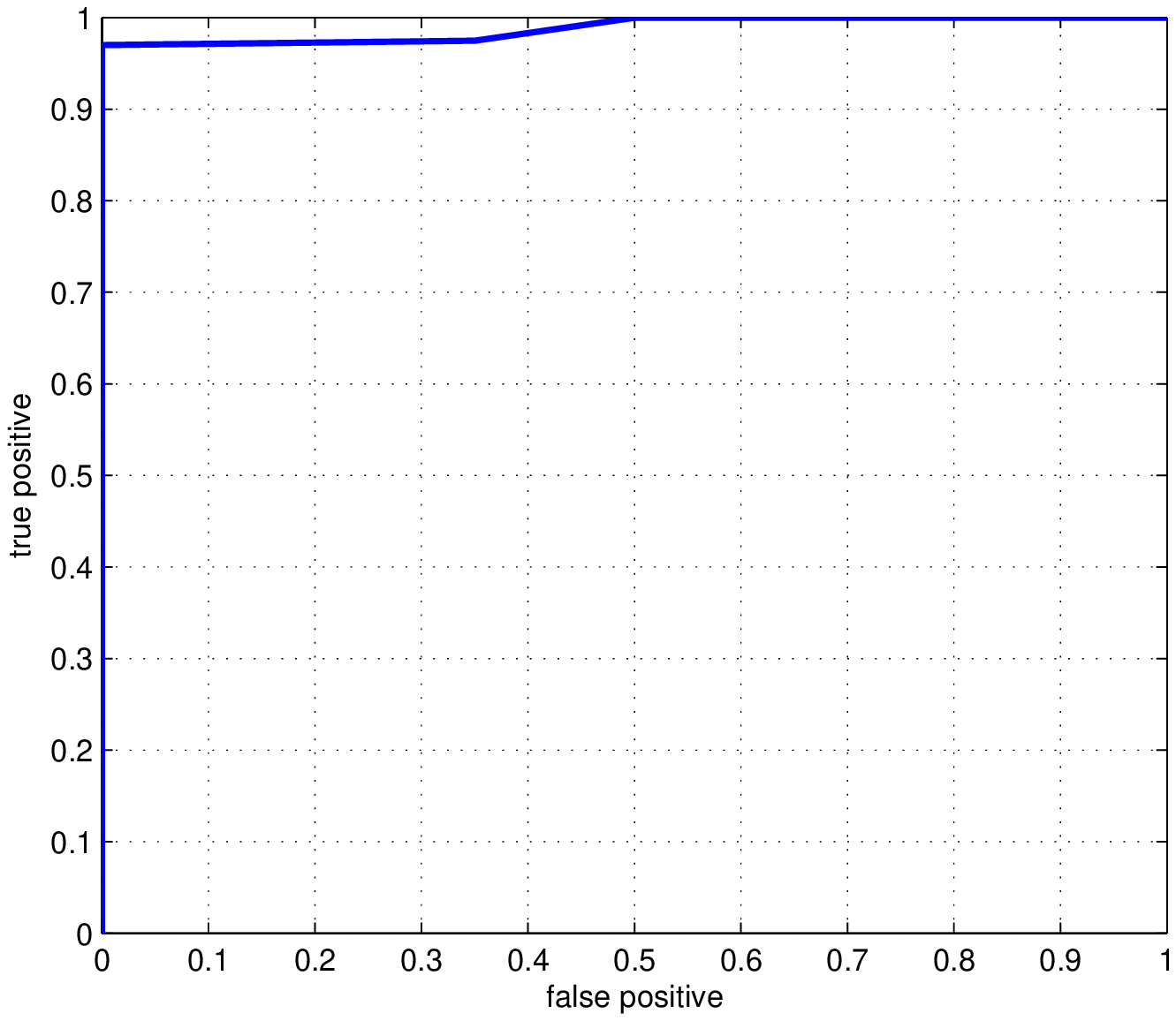}}
\put(0.31,0){\includegraphics[scale = 0.33]{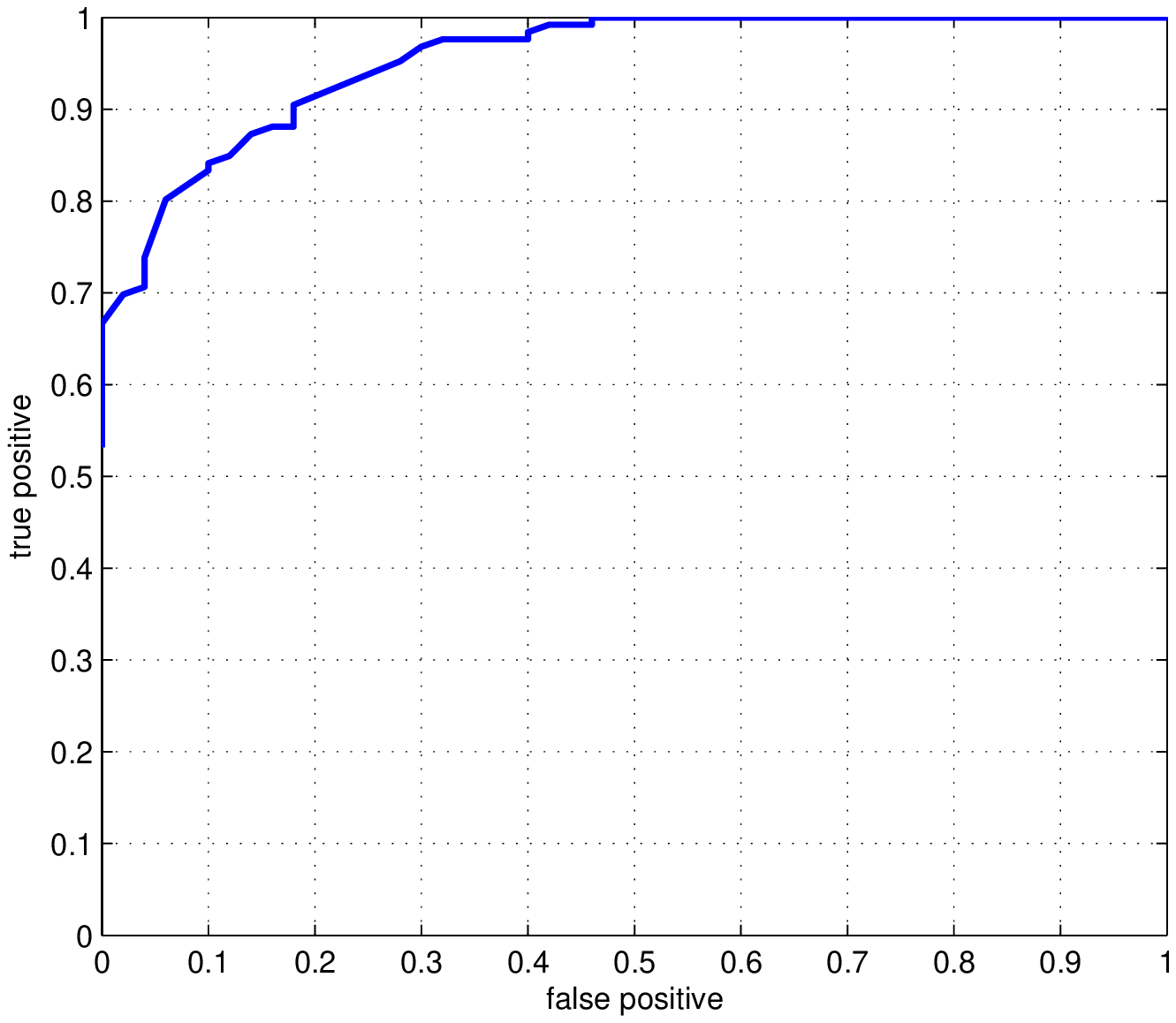}}
\put(0.65,0){\includegraphics[scale = 0.33]{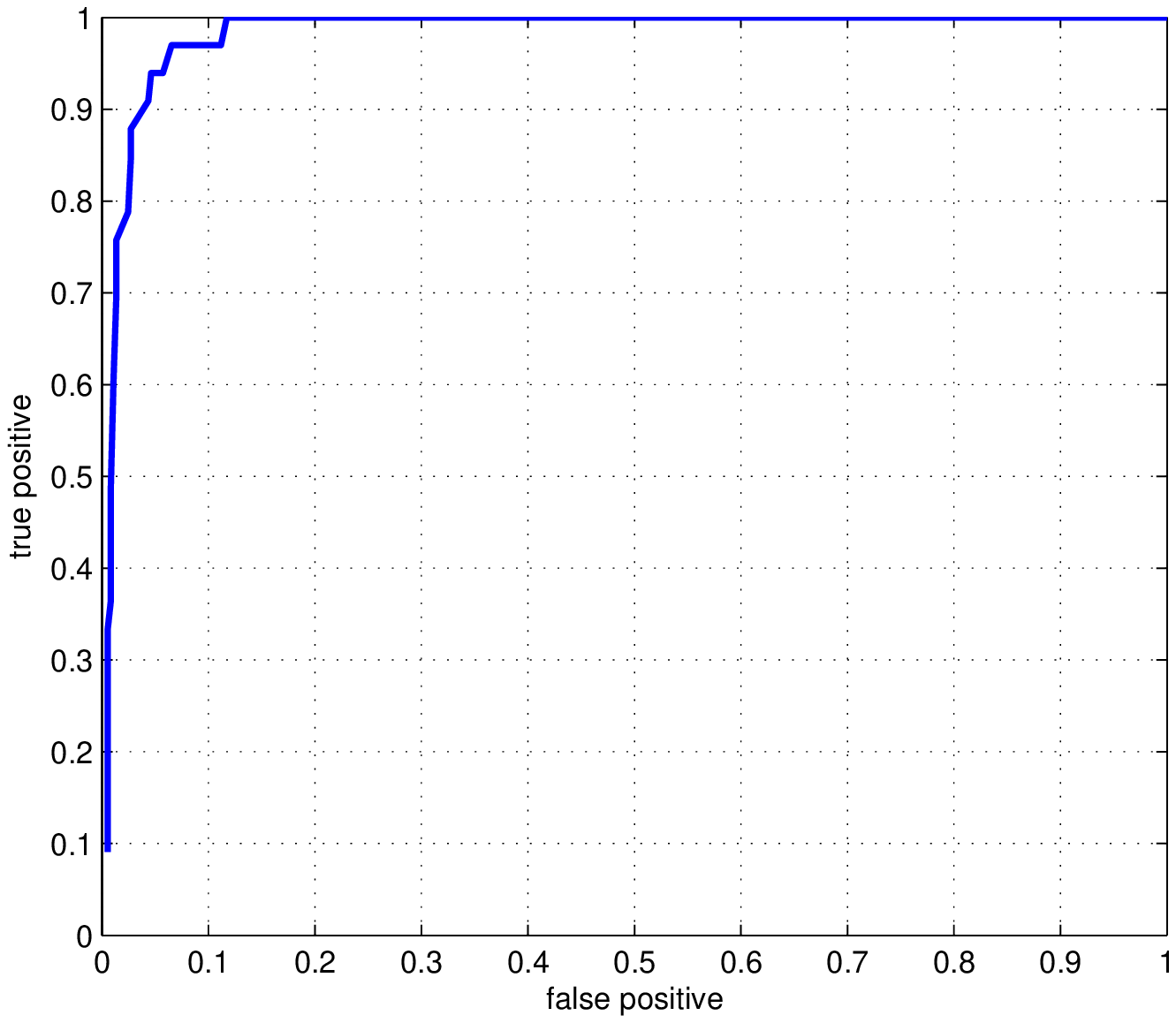}}
\put(-0.03,-.01){\makebox[0.33\textwidth]{(a) {\tt Wine}}}
\put(0.31,-0.01){\makebox[0.33\textwidth]{(b) {\tt Ionosphere}}}
\put(0.65,-0.01){\makebox[0.33\textwidth]{(c) {\tt USPS}}}
\end{picture}
\caption{\small \sl ROC curves on real datasets via {\sc LPE}; (a) {\tt
Wine} dataset with $D=13, n=39, \epsilon=0.9$; (b) {\tt Ionosphere}
dataset with $D=34,n=175,K=9$; (c) {\tt USPS} dataset with
$D=256,n=400,K=9$.}\label{fig:real}
\end{figure*}

The ROC curves of these three datasets are shown in
Figure~\ref{fig:real}. In {\tt Wine} dataset, the dimension of the
feature space is $13$. The training set is composed of $39$ data
points and we apply the $\epsilon$-LPE algorithm with $\epsilon =
0.9$. The test set is a mixture of 20 nominal points and $158$
anomaly points (ground truth). In {\tt Ionosphere} dataset, the
dimension of the feature space is $34$. The training set is composed
of $175$ data points and we apply the $K$-LPE algorithm with $K=9$.
The test set is a mixture of $50$ nominal points and $126$ anomaly
points (ground truth). In {\tt USPS} dataset, the dimension of the
feature space is $16\times 16=256$. The training set is composed of
$400$ data points and we apply the $K$-LPE algorithm with $K=9$. The
test set is a mixture of $367$ nominal points and $33$ anomaly
points (ground truth).

For comparison purposes we note that for the {\tt USPS} data set by
setting $\alpha = 0.5$ we get empirical false-positive $6.1\%$ and
empirical false alarm rate $5.7\%$ (In contrast $FP=7\%$ and
$FA=9\%$ with $\nu=5\%$ for OC-SVM as reported in \cite{Scholkopf}).
Practically we find that $K$-LPE is more preferable to
$\epsilon$-LPE due to easiness of choosing the parameter $K$. We
find that the value of $K$ is relatively independent of dimension
$d$. As a rule of thumb we
found that setting $K$ around $n^{2/5}$ was generally effective. %It is more
%difficult to choose $\epsilon$ as it is potentially dependent on
%both $n$ and $d$. The proof of Lemma \ref{lem:mean} suggests that
%$\epsilon$ should be approximately chosen to be $\sqrt{\frac{d}{2\pi
%e}}n^{-\frac{3}{5d}}$ as a rule of
%thumb.

\section{Conclusion}\label{sec:dis}
In this paper, we proposed a novel non-parametric adaptive anomaly detection algorithm which leads to a computationally efficient solution with provable optimality guarantees. Our algorithm takes a K-nearest neighbor graph as an input and produces a score for each test point. Scores turn out to be empirical estimates of the volume of minimum volume level sets containing the test point. While minimum volume level sets provide an optimal characterization for anomaly detection, they are high dimensional quantities and generally difficult to reliably compute in high dimensional feature spaces. Nevertheless, a sufficient statistic for optimal tradeoff between false alarms and misses is the volume of the MV set itself, which is a real number. By computing score functions we avoid computing high dimensional quantities and still ensure optimal control of false alarms and misses.
The computational cost of our algorithm scales linearly in dimension and quadratically in data size.

\footnotesize
\bibliographystyle{IEEEtran}
\bibliography{nips2009}

%
%References follow the acknowledgments. Use unnumbered third level heading for
%the references. Any choice of citation style is acceptable as long as you are
%consistent. It is permissible to reduce the font size to `small' (9-point)
%when listing the references. {\bf Remember that this year you can use
%a ninth page as long as it contains \emph{only} cited references.}
%
%\small{
%[1] Alexander, J.A. \& Mozer, M.C. (1995) Template-based algorithms
%for connectionist rule extraction. In G. Tesauro, D. S. Touretzky
%and T.K. Leen (eds.), {\it Advances in Neural Information Processing
%Systems 7}, pp. 609-616. Cambridge, MA: MIT Press.
%
%[2] Bower, J.M. \& Beeman, D. (1995) {\it The Book of GENESIS: Exploring
%Realistic Neural Models with the GEneral NEural SImulation System.}
%New York: TELOS/Springer-Verlag.
%
%[3] Hasselmo, M.E., Schnell, E. \& Barkai, E. (1995) Dynamics of learning
%and recall at excitatory recurrent synapses and cholinergic modulation
%in rat hippocampal region CA3. {\it Journal of Neuroscience}
%{\bf 15}(7):5249-5262.
%}

\end{document}